\documentclass[11pt]{article}
\usepackage{mathrsfs,amsmath,amsfonts,amssymb,bm,bbm,eufrak,dsfont,pifont,amscd,
stmaryrd,euscript,amsthm,appendix,color,epsfig,xr,yfonts,tikz,verbatim}
\usepackage[round]{natbib}
\PassOptionsToPackage{numbers,sort&compress}{natbib}
\usepackage{graphicx,color,authblk}
\usepackage{amsfonts}
\usepackage{bbm}
\usepackage{amsmath}
\usepackage{tikz}
\usepackage{hyperref}
\newtheorem{theorem}{Theorem}[section]
\newtheorem{lemma}[theorem]{Lemma}
\newtheorem{definition}[theorem]{Definition}
\newtheorem{proposition}[theorem]{Proposition}
 
\newtheorem{corollary}[theorem]{Corollary} 
\newtheorem{assumption}[theorem]{Assumption} 
\theoremstyle{remark}
\newtheorem{remark}{Remark}

\newcommand{\oh}{\otimes_{\mathcal{H}}}

\newcommand{\HSH}{{\Cal{L}^2(\mathcal{H})}}

\newcommand{\OPH}{{\Cal{L}^\infty(\mathcal{H})}}
\newcommand{\OPRn}{{\Cal{L}^\infty(\mathbb{R}^n)}}

\newcommand{\Hk}{\mathcal{H}}

\newcommand{\Var}{\text{Var}}

\newcommand{\inner}[2]{\left\langle #1,#2\right\rangle}
\newcommand{\norm}[1]{\left\lVert#1\right\rVert}

\newcommand{\R}{\mathbb{R}}

\newcommand{\X}{\mathcal{X}}

\newcommand{\E}{\mathbb{E}}

\newcommand{\Pb}{\mathbb{P}}

\newcommand{\bb}{\mathbb}

\newcommand{\Cal}{\mathcal}

\newcommand{\HS}{\mathcal{L}^2}
\newcommand{\Tr}{\mathcal{L}^1}

\title{Gain with no Pain: Efficient Kernel-PCA by Nystr\"om Sampling}

\author[1]{Nicholas Sterge\thanks{nzs5368@psu.edu.edu}}
\author[1]{Bharath Sriperumbudur\thanks{bks18@psu.edu}}
\author[2]{Lorenzo Rosasco\thanks{lrosasco@mit.edu.edu}}
\author[3]{Alessandro Rudi\thanks{alessandro.rudi@inria.fr}}
\affil[1]{Department of Statistics, Pennsylvania State University}
\affil[2]{LCSL, Massachusetts Institute of Technology \& Istituto Italiano di Tecnologia \& DIBRIS, Universita' degli Studi di Genova}
\affil[3]{SIERRA Project-Team, INRIA and \'{E}cole-Normale Sup\'{e}rieure, PSL Research University, Paris, France}

\date{}
\begin{document}

\maketitle

\begin{abstract}
In this paper, we propose and study a Nystr\"om based approach to efficient large scale kernel principal component analysis  (PCA). The latter is a natural nonlinear extension of classical PCA based on considering a nonlinear feature map or the  corresponding kernel. Like other kernel approaches, kernel PCA enjoys good mathematical and statistical properties but, numerically, it scales poorly with the sample size. Our analysis shows that Nystr\"om sampling greatly improves computational efficiency without incurring any loss of statistical accuracy. While similar effects have been observed in supervised learning, this is the first such result for PCA. Our theoretical findings, which are also illustrated by numerical results, are based on a combination of analytic and concentration of measure techniques. Our study is more broadly motivated by the question of understanding the  interplay  between statistical and computational requirements for learning.
\end{abstract}

\section{Introduction}\label{Sec:introduction}
Achieving good statistical accuracy under budgeted computational resources is a central theme in modern machine learning \citep{boubot}. Indeed, the problem of understanding the interplay and trade-offs between statistical and computational requirements has recently received much attention. Nonparametric learning, and in particular kernel methods, have provided a natural framework to pursue these questions, see e.g.\citep{musco17, Rudi-15, alaoui14, bach,  calandriello, orabona}. On the one hand, these methods are developed in a sound mathematical setting and their statistical properties are well studied. On the other hand, from a numerical point of view, they scale poorly to large scale problems, and hence improved computational efficiency is of particular interest. 

While initial studies have mostly focused on approximating kernel matrices \citep{Drineas-05,Gittens-13,Jin-13}, recent results
have highlighted the importance of considering downstream learning tasks, if the interplay between statistics and computation is of interest. In particular, results in supervised learning have shown there are regimes where computational gains can be achieved with no loss of statistical accuracy \citep{Rudi-15, Rudi-17}. A basic intuition is that approximate computations provide a form of implicit regularization, hence memory and time requirements can be tailored to statistical accuracy allowed by the data \citep{Rudi-15}.  To which extent  similar effects occur beyond supervised learning is unclear. Indeed, the only result in this direction was recently shown for kernel k-means in \citep{calandriello}. 

In this paper, we consider one of the most basic unsupervised approaches, namely PCA, or rather its nonlinear version, that is kernel PCA \citep{Scholkopf-98}. We develop a computationally efficient approximate kernel PCA algorithm using the Nystr\"om method \citep{Williams-01} with $m$ sub-samples (NY-KPCA) and show its time complexity to be $O(nm^2+m^3)$ with a space complexity of $O(m^2)$, in contrast to $O(n^3)$ and $O(n^2)$ time and space complexities of KPCA, where $n$ is the sample size. 
Our main contribution is the analysis of NY-KPCA in terms of finite sample bounds on the reconstruction error of the corresponding $\ell$-dimensional eigenspace (see Theorem \ref{main theorem} and related Corollaries~\ref{poly decay corollary} and \ref{exp decay corollary}). In particular, we show that NY-KPCA can achieve  the same error of KPCA  with $m<n$, 
thereby demonstrating computational gains can occur at no statistical loss. Moreover, we show that adaptive sampling using leverage scores \citep{alaoui14} can lead to further gains. More  precisely, we show that the requirement on $m$ varies between $(\log n)^2$ and $n^\theta\log n$ ($\theta<1$) depending on the size of $\ell$, the rate of decay of eigenvalues of the covariance operator and the type of subsampling. Finally, we also present some simple numerical results to corroborate our theoretical results.

We note that some recent papers, see \citep{Sriperumbudur_Sterge,streaming_kpca}, have considered the problem of deriving efficient kernel PCA approximations using random features \citep{Rahimi-08a}.  However,  the notion of reconstruction error considered in these works is different from that of KPCA \citep{Shawe-Taylor-05,Blanchard-07}. The reason for a different notion of reconstruction error is to handle certain technicalities that arise in random feature approximation. As a consequence, these results are not directly comparable to our current work and KPCA. In contrast, our results based on Nystr\"{o}m approximation are directly comparable to that of  KPCA, wherein we show that the proposed NY-KPCA has similar statistical behavior but better computational complexity than KPCA.



The paper is organized as follows. Relevant notations and definitions are collected in Section~\ref{Sec:notation}. Section~\ref{Sec:Prelim} provides preliminaries on KPCA 
along with the list of assumptions that will be used throughout the paper. Approximate KPCA using Nystr\"{o}m method is presented in Section~\ref{ny-kpca subsection} and the main results of computational vs. statistical tradeoff for NY-KPCA are presented in Section~\ref{Sec:Results}. Missing proofs of the results are provided in the appendix.

\section{Definitions and Notation}\label{Sec:notation}
\indent For $\bm{a}:=(a_1,\ldots,a_d)\in\bb{R}^d$ and $\bm{b}:=(b_1,\ldots,b_d)\in\bb{R}^d$ define $\Vert
\bm{a}\Vert_2:=\sqrt{\sum^d_{i=1}a^2_i}$ and $\langle \bm{a},\bm{b}\rangle_2:=\sum^d_{i=1}a_ib_i$. 
$\bm{a}\otimes_2 \bm{b}:=\bm{a}\bm{b}^\top$ denotes the tensor product of $\bm{a}$ and $\bm{b}$. $\mathbf{I}_n$ denotes an $n\times n$ identity matrix. 
$a\wedge b:=\min(a,b)$ and $a\vee b:=\max(a,b)$. $[n]:=\{1,\ldots,n\}$ for $n\in\bb{N}$. For constants $a$ and $b$, $a\lesssim b$ (\emph{resp.} $a\gtrsim b$) denotes that there exists a positive constant $c$ (\emph{resp.} $c'$) such that $a\le cb$ (\emph{resp.} $a\ge c'b$). For a random variable $A$ with law $P$ and a constant $b$, $A\lesssim_p b$ denotes that for any $\delta>0$, there exists a positive constant $c_\delta<\infty$ such that $P(A\le c_\delta b)\ge \delta$.

For $x,y\in
H$, a Hilbert space, $x\otimes_{H} y$ is an element of the tensor product space
$H\otimes H$ which can also be seen as an operator from $H$ to $H$ as
$(x\otimes_{H} y)z=x\langle y,z\rangle_{H}$ for any $z\in H$. $\alpha\in\bb{R}$ is called an \emph{eigenvalue} of a bounded self-adjoint operator $S$ 
if there exists an $x\ne 0$ such that $Sx=\alpha x$ and such an $x$ is called
the \emph{eigenvector}/\emph{eigenfunction} of $S$ and $\alpha$. An eigenvalue is said to be \emph{simple} if it has multiplicity one. For an operator $S:H\rightarrow H$, $\Vert S\Vert_{\Tr(H)}$,
$\Vert S\Vert_{\HS(H)}$ and $\Vert S\Vert_{\Cal{L}^\infty(H)}$ denote the trace, Hilbert-Schmidt and operator norms of $S$, respectively.

\section{Kernel PCA by Nystr\"om Sampling}\label{Sec:Prelim}

In this section, we review kernel principal component analysis  (KPCA) \citep{Scholkopf-98} in population and empirical settings and introduce approximate kernel PCA using Nystr\"om approximation.  We assume the following for the rest of the paper:

\begin{assumption}\label{first kern assum}
$\X$ is a separable topological space and $(\Hk,k)$ is a separable RKHS of real-valued functions $\X$ with a bounded, continuous, strictly positive definite kernel $k$ satisfying $\sup_{x\in\X}\,k(x,x)=:\kappa<\infty$.
\end{assumption}

\subsection{KPCA and Empirical KPCA}\label{Sec:kpca}

Let $X$ be a zero-mean random variable with law $\Pb$ defined on $\Cal{X}$. When $\Cal{X}=\bb{R}^d$, classical PCA \citep{Jollife-86} finds $\mathbf{a}\in\R^d$ such that $\text{Var}\left[\inner{\mathbf{a}}{X}_2\right]$ is maximized, with the constraint $\norm{\mathbf{a}}_2=1$.  Defining $C:=\bb{E}_{X\sim\Pb}[XX^\top]$, the solution is simply the unit eigenvector of $C$ corresponding to its largest eigenvalue. In practice, PCA is computed by replacing $C$ with an empirical approximation $C_n= \frac 1 n \sum_{i=1}^n X_i X_i^\top$ based on a sample $X_1, \dots, X_n$. Kernel PCA extends this idea to an RKHS, $\Hk$ defined on $\Cal{X}$, by finding $f\in\Hk$ with unit norm such that $\Var[f(X)]$ is maximized. Since $\Var[f(X)]=\langle f,Cf\rangle_{\Hk}$ assuming $\bb{E}[f(X)]=0$ for all $f\in\Hk$, we have $f^*=\arg\sup\{\langle f,Cf\rangle_\Hk:\norm{f}_\Hk=1\}$ where $C$ is the (uncentered) covariance operator on $\Hk$ defined as
\begin{equation}\label{C defin}
    C:=\int_\X k(\cdot,x)\oh k(\cdot,x)\,d\Pb(x).
\end{equation}
The boundedness of $k$ in Assumption~\ref{first kern assum} ensures that $C$ is trace class and thus compact.  Since $C$ is positive and self-adjoint, the spectral theorem \citep{Reed-80} gives
\begin{equation}\label{C spectral}
    C=\sum_{i\in I}\lambda_i\phi_i\oh\phi_i,
\end{equation}
where $(\lambda_i)_{i\in I}\subset\R^+$ are the eigenvalues and $(\phi_i)_{i\in I}$ are the orthonormal system of eigenfunctions that span $\overline{\Cal{R}(C)}$ with index set $I$ either being finite or countable, in which case $\lambda_i\rightarrow 0$ as $i\rightarrow\infty$.  The solution to the KPCA problem is thus the eigenfunction of $C$ corresponding to its largest eigenvalue. We make the following simplifying assumption for ease of presentation.
\begin{assumption}\label{C eigen assum}
The eigenvalues $(\lambda_i)_{i\in I}$ of $C$ are simple, positive, and w.l.o.g.~they satisfy a decreasing rearrangement, i.e., $\lambda_1>\lambda_2,\ldots$\vspace{-2mm}
\end{assumption}
\noindent Assumption \ref{C eigen assum} ensures that $(\phi_i)_{i\in I}$ form an orthonormal basis and the eigenspace corresponding to each $\lambda_i$ is one-dimensional.  This means the orthogonal projection operator onto the $\ell$-eigenspace of $C$, i.e. span$\{(\phi_i)_{i=1}^\ell\}$, is given by\vspace{-2mm}
\begin{equation}\label{kpca projector}
P^\ell(C)=\sum_{i=1}^\ell\phi_i\oh\phi_i.
\end{equation}
The above construction corresponds to population version of KPCA when the data distribution $\Pb$  is known.  If $\Pb$ is unknown and the knowledge of $\Pb$ is available only through the training set $\left(X_i\right)_{i=1}^n\stackrel{i.i.d.}{\sim}\Pb$, then KPCA cannot be carried out as $C$ depends on $\Pb$. Therefore, an approximation to $C$ is used to perform KPCA. Most commonly, this approximation is chosen to be the empirical estimator of $C$ defined as 
\begin{equation}\label{C_n defin}
    C_n=\frac{1}{n}\sum_{i=1}^nk(\cdot,X_i)\oh k(\cdot,X_i)
\end{equation}
resulting in empirical kernel
PCA (EKPCA).  Note that $C_n$ is a finite rank, positive, and self-adjoint operator. Thus the spectral theorem \citep{Reed-80} yields
\begin{equation}\label{C_n spectral}
    C_n=\sum_{i=1}^{n}\hat{\lambda}_i\hat{\phi}_i\oh\hat{\phi}_i,
\end{equation}
where $(\hat{\lambda}_i)_{i=1}^{n}\subset\R^+$ and $(\hat{\phi}_i)_{i=1}^{n}\subset\Hk$ are the eigenvalues and eigenfunctions of $C_n$.  Similar to Assumption~\ref{C eigen assum}, we assume the following:
\begin{assumption}\label{C_n eigen assum}
$\emph{rank}(C_n)=n$, the eigenvalues $(\hat{\lambda}_i)_{i=1}^{n}$ of $C_n$ are simple and w.l.o.g. they satisfy a decreasing rearrangement, i.e., $\hat{\lambda}_1\ge\hat{\lambda}_2\ge\ldots$.\vspace{-2mm}
\end{assumption}
The eigensystem $(\hat{\lambda}_i,\hat{\phi}_i)_{i=1}^n$ of $C_n$ can be obtained by solving an $n$-dimensional system involving the eigendecomposition of the Gram matrix $\mathbf{K}=[k(X_i,X_j)]_{i,j\in[n]},$ which scales as $O(n^3)$ \citep{Scholkopf-98}. 
In particular, the eigenvalues of $\mathbf{K}$ are related to those of $C_n$ as $\lambda_i(\mathbf{K})=n\hat{\lambda}_i$. Moreover, if $\mathbf{u}_i$ is an orthonormal eigenvector of $\mathbf{K}$
 corresponding to the eigenvalue  $\lambda_i(\mathbf{K})$, then it holds for all $x\in {\mathcal X}$,
 \begin{equation}\label{eq:rep}
 \phi_i(x)= \frac{1}{\sqrt{n\hat{\lambda}_i}}\sum_{j=1}^n k(x, x_j)u_{i,j}.
 \end{equation}
 The above result proven in 
\citep{Scholkopf-01}   can be seen as a representer theorem \citep{Kimeldorf-71} for KPCA. Finally,  note that,  for some $\ell\le n$, the orthogonal projection operator onto $\text{span}\{(\hat{\phi}_i)_{i=1}^\ell\}$ is given by
\begin{equation}\label{ekpca projector}
    P^\ell(C_n)=\sum_{i=1}^\ell\hat{\phi}_i\oh\hat{\phi}_i.
\end{equation}
\subsection{Approximate  Kernel PCA using Nystr\"om Method}\label{ny-kpca subsection}
For large sample sizes, since performing KPCA is computationally intensive, various approximation schemes that has been explored in the kernel machine literature can be deployed to speed up EKPCA. Recently, one such approximation involving random Fourier features has been studied by \citet{Sriperumbudur_Sterge} and \citet{streaming_kpca} to speed EKPCA while maintaining its statistical performance. In this paper, we explore the popular Nystr\"om approximation \citep{Williams-01,Drineas-05} to speed up EKPCA and study the trade-offs between  computational gains  and  statistical accuracy. 
The general idea in Nystr\"om method is to obtain a low-rank approximation to the Gram matrix $\mathbf{K}$, and replace $\mathbf{K}$ by this approximation in kernel algorithms, resulting in computational speedup. Since $\mathbf{K}$ is related to $C_n$ (as discussed in Section~\ref{Sec:kpca}), Nystr\"{o}m method can also be seen as obtaining a low rank approximation to $C_n$, which is what we exploit in obtaining a Nystr\"om approximate KPCA.  It follows from~\eqref{eq:rep} 
 that the eigenfunctions of $C_n$ lie in the space
$$\Hk_n=\left\{f\in\Hk\,\Big{|}\,f=\sum_{i=1}^n\alpha_ik(\cdot,X_i),\alpha_1,...,\alpha_n\in\R\right\}.$$
Therefore, it can be seen that EKPCA is a solution to the following problem
$$\arg\sup\left\{\inner{f}{C_nf}_\Hk:f\in\Hk_n,\,\norm{f}_\Hk=1\right\},$$
assuming $\mathbf{K}$ is invertible\footnote{The existence of $\mathbf{K}^{-1}$ is guaranteed by strict positive definiteness of $k$, provided all $X_i$ in the training set are unique.}. Extending this representation, we propose Nystr\"om KPCA (NY-KPCA) as a solution to the following problem:
\begin{equation}\label{NY-KPCA problem defin}
\arg\sup\left\{\inner{f}{C_nf}_\Hk:f\in\Hk_m,\,\norm{f}_\Hk=1\right\},
\end{equation}
where $$\Hk_m=\left\{f\in\Hk\,\Big|\,f=\sum_{i=1}^m\alpha_ik(\cdot,\tilde{X}_i),\alpha_1,...,\alpha_m\in\R\right\}$$
is a low-dimensional subspace of $\Hk_n$ and $\{\tilde{X}_1,...,\tilde{X}_m\}$ is a subset of the training set with $\tilde{X}_i$'s being distinct. Basically, we are considering a plain Nystr\"{o}m approximation where the points $\{\tilde{X}_1,\ldots,\tilde{X}_m\}$ are sampled uniformly without replacement from $\{X_1,\ldots,X_n\}$, however, other subsampling methods are possible, see Section \ref{lev scores subsection}. The following result, which is proved in the supplement (see Section~\ref{sec:ny-kpca soln}), shows that the solution to (\ref{NY-KPCA problem defin}) is obtained by solving a finite dimensional linear system, which has better computational complexity than that of EKPCA. To this end, we first introduce some notation,
$\mathbf{K}_{mm}=[k(\tilde{X_i},\tilde{X_j})]_{i,j\in[m]}\text{, }\mathbf{K}_{nm}=[k(X_i,\tilde{X_j})]_{i\in[n],j\in[m]}\in\R^{n\times m}, \mathbf{K}_{mn}=\mathbf{K}_{nm}^\top.$
\begin{proposition}\label{ny-kpca soln}
Define the $m\times m$ matrix $\mathbf{M}=\mathbf{K}_{mm}^{-1/2}\mathbf{K}_{mn}\mathbf{K}_{nm}\mathbf{K}_{mm}^{-1/2}$. The solution to (\ref{NY-KPCA problem defin}) is given by
$$\hat{\phi}_{1,m}=\tilde{Z}_m^*\mathbf{K}_{mm}^{-1/2}\mathbf{u}_{1,m},$$
where $\mathbf{u}_{1,m}$ is the eigenvector of $\frac{1}{n}\mathbf{M}$ corresponding to its largest eigenvalue and $\tilde{Z}_m^*:\R^m\rightarrow\Hk,\,\bm{\alpha}\mapsto\sum_{i=1}^m\alpha_ik(\cdot,\tilde{X}_i).$
\end{proposition}
The cost of computing $\mathbf{M}$ is $O(nm^2+m^3)$ and the cost of computing its eigendecomposition is $O(m^3)$.  Thus, for $m<n$, the cost of NY-KPCA scales as $O(nm^2)$, faster than the $O(n^3)$ cost of EKPCA. Define 
\begin{equation}\label{nystrom approx gram}
    \tilde{\mathbf{K}}:=\mathbf{K}_{nm}\mathbf{K}_{mm}^{-1}\mathbf{K}_{mn},
\end{equation}
which is called the Nystr\"om approximation \citep{Williams-01,Drineas-05} to the Gram matrix $\mathbf{K}$.  It is easy to verify that $\mathbf{M}$ and $\tilde{\mathbf{K}}$ have same eigenvalues since $\mathbf{M}=\mathbf{K}_{mm}^{-1/2}\mathbf{K}_{mn}\left(\mathbf{K}_{mm}^{-1/2}\mathbf{K}_{mn}\right)^\top$ and $\tilde{\mathbf{K}}=\left(\mathbf{K}_{mm}^{-1/2}\mathbf{K}_{mn}\right)^\top\mathbf{K}_{mm}^{-1/2}\mathbf{K}_{mn}$, and rank$(\mathbf{M})=\text{rank}(\tilde{\mathbf{K}})$. Therefore we work with $\tilde{\mathbf{K}}$ and make the following assumption on its eigenvalues.
\begin{assumption}\label{nystrom eigen assum}
$\emph{rank}(\tilde{\mathbf{K}})=m$. The eigenvalues $(\hat{\lambda}_{i,m})_{i=1}^m$ of $\frac{1}{n}\tilde{\mathbf{K}}$ are simple and w.l.o.g. they satisfy a decreasing rearrangement, i.e., $\hat{\lambda}_{1,m}>\hat{\lambda}_{2,m}\ldots>\hat{\lambda}_{m,m}$.
\end{assumption}
The symmetry of $\mathbf{M}$ guarantees orthonormality of $(\mathbf{u}_{i,m})_i$, and the orthonormality of $(\hat{\phi}_{i,m})_i$ follows. For some $\ell\le m$, the orthogonal projector onto span$\{\hat{\phi}_{i,m}\}_{i=1}^\ell$ is given by
\begin{equation}\label{NY-KPCA projector}
P^\ell_m(C_n)=\sum_{i=1}^\ell\hat{\phi}_{i,m}\oh\hat{\phi}_{i,m}.
\end{equation}
One may ask if $\hat{\phi}_{i,m}$ are eigenfunctions of some operator on $\Hk$. Denote $P_m$ as the orthogonal projector onto $\Hk_m$. It is simple to verify \citep[Theorem 2]{Rudi-15} that $P_m=\tilde{Z}_m^*\mathbf{K}_{mm}^{-1}\tilde{Z}_m$ and that $\left(\hat{\lambda}_{i,m},\hat{\phi}_{i,m}\right)$ are the orthonormal eigenfunctions of $P_mC_nP_m$, i.e.,
\begin{equation}\label{P_mC_nP_m eigen}
    P_mC_nP_m\hat{\phi}_{i,m}=\hat{\lambda}_{i,m}\hat{\phi}_{i,m}\text{   for all   }i\in[m].
\end{equation}
Therefore, we may think of $P_mC_nP_m$ as a low-rank approximation to $C_n$.

\subsubsection{Approximate Leverage Scores}\label{lev scores subsection}
In the above discussion on Nystr\"om KPCA, $\tilde{\mathbf{X}}:=\{\tilde{X}_1,\ldots,\tilde{X}_m\}$ is a subset of the training set $\mathbf{X}:=\{X_1,\ldots,X_n\}$ with the entries of $\tilde{\mathbf{X}}$ being sampled uniformly without repetition from $\mathbf{X}$. As an alternative to uniform sampling, $\tilde{\mathbf{X}}$ can be sampled according to the leverage score distribution \citep{Alaoui-15, Drineas-12, Cohen-15}.  For any $s>0$, the leverage scores associated with the training data $\mathbf{X}$ are defined as
$$(l_i(s))_{i=1}^n\,,\quad l_i(s)=[\mathbf{K}(\mathbf{K}+ns\mathbf{I}_n)^{-1}]_{ii}\,,i\in[n]$$
with the leverage score distribution being $p_i(s)=\frac{l_i(s)}{\sum^n_{i=1}l_i(s)}$ according to which $\mathbf{X}$ can be sampled independently with replacement to achieve $\tilde{\mathbf{X}}$. Since the leverage scores are computationally intensive to compute, usually, they are approximated and one such approximation is $T$-approximate leverage scores.
\begin{definition}\label{def:lev scores}($T$-approximate leverage scores) For a given $s>0$, let $(l_i(s))_{i=1}^n$ be the leverage scores associated with the training data $\{X_1,...,X_n\}$.  Let $\delta>0$, $s_0>0$, and $T\ge1$.  $(\hat{l}_i(s))_{i=1}^n$ are $T$-approximate leverage scores, with confidence $\delta$, if the following holds with probability at least $1-\delta$:
$$\frac{1}{T}l_i(s)\le\hat{l}_i(s)\le\,T l_i(s),\quad\forall i\in[n],\quad s>s_0.$$
\end{definition}
Given $T$-approximate leverage scores for $s>s_0$, $\tilde{\mathbf{X}}$ can be obtained by sampling $\mathbf{X}$ with replacement according to the sampling distribution $\hat{p}_i(s)=\hat{l}_i(s)/\sum_{i=1}^n\hat{l}_i(s)$. Having obtained $\tilde{\mathbf{X}}$, (\ref{NY-KPCA problem defin}) can be solved exactly as in Proposition \ref{ny-kpca soln}. We refer to this method as approximate leverage score (ALS) Nystr\"om subsampling.
\section{Computational vs.~Statistical Trade-Off: Main Results}\label{Sec:Results}

As shown in the earlier section, Nystr\"om kernel PCA approximates the solution to empirical kernel PCA with less computational expense. In this section, we explore whether this computational saving is obtained at the expense of statistical performance. As in \citet{Sriperumbudur_Sterge}, we measure the statistical performance of KPCA, EKPCA, and NY-KPCA in terms of reconstruction error. In linear PCA, the reconstruction error, given by \begin{equation}\label{rec}
    \bb{E}_{X\sim\Pb}\norm{\left(I-P^\ell(C)\right)X}_2^2,\end{equation}
is the error involved in reconstructing a random variable $X$ by projecting it onto the $\ell$-eigenspace (i.e., span of the top-$\ell$ eigenvectors) associated with its covariance matrix, $C=\bb{E}[XX^\top]$ through the orthogonal projection operator $P^\ell(C)$.
Clearly, the error is zero when $\ell=d$. The analog of the  reconstruction error in KPCA, as well as EKPCA and NY-KPCA, can be similarly stated in terms of their projection operators, (\ref{kpca projector}), (\ref{ekpca projector}), and (\ref{NY-KPCA projector}) as follows. For any orthogonal projection operator $P:{\mathcal H}\to {\mathcal H}$, define the reconstruction error as 
$$
R(P):=\bb{E}_{X\sim\Pb}\norm{\left(I-P\right)k(\cdot,X)}_\Hk^2.$$
For the linear kernel this exactly the reconstruction error of PCA. In the following,  we often make use of the following identity
\begin{equation}
R(P)=\Vert (I-P)C^{1/2}\Vert^2_{\HSH},\label{Eq:equiv}\end{equation}
 for which we report a proof in the supplement (see Section \ref{sec:rewrite recons err}). Based on this definition, the reconstruction error in KPCA, EKPCA and NY-KPCA are given by
\begin{equation}R_{C,\ell}:=R(P^\ell(C)),\,\,R_{C_n,\ell}:=R(P^\ell(C_n)),\,\,\text{and}\,\,R^{nys}_{C_n,\ell}:=R(P^\ell_m(C_n))\label{Eq:errors}\end{equation}
respectively.
The following theorem, proved in the supplement (see Section \ref{main theorem proof}), provides finite-sample bounds on the reconstruction error associated with NY-KPCA, under both uniform and approximate leverage score subsampling, from which convergence rates may be obtained.
\begin{theorem}\label{main theorem}
Suppose Assumptions \ref{first kern assum}-\ref{nystrom eigen assum} hold. For any $t>0$, define $\Cal{N}_C(t)=\emph{tr}((C+tI)^{-1}C)$ and $\Cal{N}_{C,\infty}(t)=\sup_{x\in\Cal{X}}\langle k(\cdot,x),(C+tI)^{-1}k(\cdot,x)\rangle_\Cal{H}$. 
Then the following hold:\vspace{2mm}\\
(i)  Suppose $n>3$, $0<\delta<1$, $\frac{9\kappa}{n}\log\frac{n}{\delta}\le t\le \lambda_1$, and $m\ge\left(67\vee5\Cal{N}_{C,\infty}(t)\right)\log\frac{4\kappa}{t\delta}$. Then, for plain Nystr\"om subsampling:
\begin{equation}\Pb^n\left\{(X_i)^n_{i=1}:R_{C_n,\ell}^{nys}\le \Cal{N}_C(t)\left(6\lambda_\ell+42t\right)\right\}\ge 1-2\delta.\label{Eq:nykpca-1}\end{equation}
(ii) For $0<\delta<1$, suppose there exists $T\ge1$ such that  $(\hat{l}_i(s))_{i=1}^n$ are $T-$approximate leverage scores with confidence $\delta$ for any $t\ge\frac{19\kappa}{n}\log\frac{2n}{\delta}$. Assume approximate leverage score Nystr\"om subsampling is used with 
$$t=\min\left\{\frac{19\kappa}{n}\log\frac{2n}{\delta}\le t\le\lambda_1\,\Big{|}\,78T^2\Cal{N}_C(t)\log\frac{8n}{\delta}\le m\right\}.$$  
If $n\ge1655\kappa+223\kappa\log\frac{2\kappa}{\delta}$ and $m\ge334\log\frac{8n}{\delta}$, then
\begin{equation}\Pb^n\left\{(X_i)^n_{i=1}:R_{C_n,\ell}^{nys}\le \Cal{N}_C(t)\left(6\lambda_\ell+42t\right)\right\}\ge 1-3\delta.\label{Eq:nykpca-2}\end{equation}
\end{theorem}
To understand the significance of Theorem~\ref{main theorem}, we have to compare it to the behavior of the reconstruction error associated with EKPCA, i.e., $R_{C_n,\ell}$. \citep[Theorem 3.1]{Rudi-15} showed that for $n>3$, $0<\delta<1$ and $\frac{9\kappa}{n}\log\frac{n}{\delta}\le t\le \lambda_1$, \begin{equation}\Pb^n\left\{(X_i)^n_{i=1}:R_{C_n,\ell}\le 9\Cal{N}_C(t)\left(\lambda_\ell+t\right)\right\}\ge 1-\delta.\label{Eq:ekpca}\end{equation}
Comparing \eqref{Eq:nykpca-1} and \eqref{Eq:nykpca-2} to \eqref{Eq:ekpca}, it is clear that NY-KPCA has a statistical behavior similar to that EKPCA. However, it is not obvious whether such a behavior is achieved for $m<n$, i.e., the order of dependence of $m$ on $n$ is not clear. To clarify this, in the following, we present two corollaries (proved in the supplement, see Sections~\ref{poly decay proof} and \ref{exp decay proof}) to Theorem~\ref{main theorem}, which compare the asymptotic convergence rates of $R_{C,\ell}$, $R_{C_n,\ell}$ and $R^{nys}_{C_n,\ell}$ under an additional assumption on the decay rate of eigenvalues of $C$. 
\begin{corollary}[Polynomial decay of eigenvalues]\label{poly decay corollary}
Suppose $\underbar{A}i^{-\alpha}\le\lambda_i\le\bar{A}i^{-\alpha}$ for $\alpha>1$ and $\underbar{A},\bar{A}\in(0,\infty)$. Let $\ell=n^{\frac{\theta}{\alpha}}$, $\theta>0$.
Then the following hold:
\\\\
$(i)$ $$n^{-\theta(1-\frac{1}{\alpha})}\lesssim R_{C,\ell}\lesssim n^{-\theta(1-\frac{1}{\alpha})};$$
$(ii)$
\[ R_{C_n,\ell}\lesssim_{\Pb^n}\begin{cases} 
      n^{-\theta(1-\frac{1}{\alpha})},\qquad\hspace{2.5mm}\theta<1 \\
      \left(\frac{\log n}{n}\right)^{1-\frac{1}{\alpha}},\qquad\theta\ge 1
   \end{cases};
\]
$(iii)$ For plain Nystr\"om subsampling: \[ R^{nys}_{C_n,\ell}\lesssim_{\Pb^n}\begin{cases} 
      n^{-\theta(1-\frac{1}{\alpha})},\qquad\hspace{2.5mm}\theta<1,\,m\gtrsim n^\theta\log n \\
      \left(\frac{\log n}{n}\right)^{1-\frac{1}{\alpha}},\qquad\theta\ge 1,\,m\gtrsim \frac{n}{\log n}\log\frac{n}{\log n}
   \end{cases};
\]
$(iv)$ For approximate leverage score Nystr\"om subsampling: \[ R^{nys}_{C_n,\ell}\lesssim_{\Pb^n}\begin{cases} 
      n^{-\theta(1-\frac{1}{\alpha})},\qquad\hspace{2.5mm}\theta<1,\,m\gtrsim n^{\frac{\theta}{\alpha}}\log n \\
      \left(\frac{\log n}{n}\right)^{1-\frac{1}{\alpha}},\qquad\theta\ge 1,\,m\gtrsim n^{\frac{1}{\alpha}}(\log n)^{1-\frac{1}{\alpha}}
   \end{cases}.
\]
\end{corollary}
\begin{remark}\label{rem1}
$(i)$ The above result shows that the reconstruction errors associated with KPCA and EKPCA have similar asymptotic behavior as long as $\ell$ does not grow to infinity too fast, i.e., $\theta<1$. On the other hand, for $\theta\ge 1$, the reconstruction error of EKPCA has slower asymptotic convergence to zero than that of KPCA. If $\ell$ grows to infinity faster with the rate controlled by $\theta$, then the variance term dominates the bias resulting in a slower convergence rate compared to that of KPCA. \vspace{1mm}
\\
$(ii)$ Comparing $(ii)$ and $(iii)$ in the above result, we note that EKPCA and NY-KPCA have similar convergence behavior as long as $m$ is large enough where the size of $m$ is controlled by the growth of $\ell$ through $\theta$. For the case of $\theta\ge 1$ in $(iii)$, we require $m\gtrsim\frac{n}{\log n}\log\frac{n}{\log n}$ 
which means asymptotically $m$ should be of the same order as $n$. On the other hand, the approximate leverage score Nystr\"{o}m subsampling gives same convergence rates as that of EKPCA but requiring far fewer samples than that for NY-KPCA with plain Nystr\"{o}m subsampling. 
These results show that for the interesting case of $\theta<1$ where EKPCA performance matches with that of KPCA, NY-KPCA also achieves similar performance, albeit with lower computational requirement. 
\end{remark}
\begin{corollary}[Exponential decay of eigenvalues]\label{exp decay corollary}
Suppose $\underbar{B}e^{-\tau i}\le\lambda_i\le\bar{B}e^{-\tau i}$ for $\tau>0$ and $\underbar{B},\bar{B}\in(0,\infty)$. Let $\ell=\frac{1}{\tau}\log n^\theta$ for $\theta>0$.
Then the following hold:
\\\\
$(i)$ $$n^{-\theta}\lesssim R_{C,\ell}\lesssim n^{-\theta};$$
$(ii)$ \[ R_{C_n,\ell}\lesssim_{\Pb^n}\begin{cases} 
      n^{-\theta}\log n ,\qquad\quad\theta<1\\

      n^{-1}(\log n)^2,\qquad\theta\ge 1
   \end{cases};
\]
$(iii)$ For plain Nystr\"om subsampling:
\[ R^{nys}_{C_n,\ell}\lesssim_{\Pb^n}\begin{cases} 
      n^{-\theta}\log n,\qquad\quad\theta<1,\,m\gtrsim n^\theta\log n\\
      n^{-1}(\log n)^2,\qquad\theta\ge 1,\,m\gtrsim \frac{n}{\log n}\log\frac{n}{\log n}
   \end{cases};
\]
$(iv)$ For approximate leverage score Nystr\"om subsampling:
\[ R^{nys}_{C_n,\ell}\lesssim_{\Pb^n}\begin{cases} 
      n^{-\theta}\log n,\qquad\quad\theta<1,\,m\gtrsim(\log n)^2 \\
      n^{-1}(\log n)^2,\qquad\theta\ge 1,\,m\gtrsim\log n\log\frac{n}{\log n}
   \end{cases}.
\]

\end{corollary}
Corollary~\ref{exp decay corollary} shares similar behavior to that Corollary~\ref{poly decay corollary} as discussed in Remark~\ref{rem1} but just that it yields faster rates since the RKHS is smooth as determined by the rate of decay of eigenvalues. In addition, the approximate leverage score Nystr\"om subsampling based KPCA requires only $(\log n)^2$ subsamples to match the performance of EKPCA resulting in substantial computational savings without any loss in statistical accuracy.

As mentioned in Section~\ref{Sec:introduction}, the above results are the first of the kind related to computational vs. statistical trade-off in kernel PCA. While \citep{Sriperumbudur_Sterge,streaming_kpca} studied similar question for kernel PCA using random features, the results are not directly comparable because of the different cost function considered in these works. To elaborate, these works also considered the reconstruction error defined in \eqref{Eq:errors} through \eqref{Eq:equiv}, however, in $L^2(\bb{P})$ norm, which is weaker than the RKHS norm. 
For classical PCA this would correspond to considering the error
$$
\E[( X^\top (I- P^\ell(C))X)^2]
$$
rather than~\eqref{rec}. This choice is made necessary by the fact that random features corresponding to a kernel, might in general not belong to the corresponding RKHS. Clearly this error choice  does not allow a direct comparison to the convergence behavior of KPCA. 

 
\section{Experiments}\label{Sec:Expts}

The goal of our experiments is to demonstrate on benchmark data that NY-KPCA achieves similar error to that of EKPCA, with significantly less computation time. For our experiments, we use the samples pertaining to the digits 2 and 5 in the MNIST handwritten digit dataset, \url{http://yann.lecun.com/exdb/mnist/}, yielding sample sizes of $n=5958$ and $n=5421$, respectively with each sample belonging to $\R^{784}$.
EKPCA is performed on each of these two digits using a Gaussian kernel, $k(\cdot,x)=\exp\{{-\sigma\norm{\cdot-x}_2^2}\},$
with $\sigma=1\times10^{-7}$ and NY-KPCA is performed with plain Nystr\"om subsampling, i.e., uniformly without replacement, for $m=$100, 500 and 1000 
Nystr\"om subsamples with 100 repetitions being performed for each $m$ to generate error bars. The reconstruction error is measured as 
$$\hat{R}(P):=\frac{1}{n}\sum_{i=1}^n\norm{k(\cdot,X_i)-Pk(\cdot,X_i)}_\Hk^2,$$
with $P:\Cal{H}\rightarrow\Cal{H}$ chosen to be $P^\ell(C_n)$ and $P^\ell_m(C_n)$ for EKPCA and NY-KPCA respectively. These quantities can be computed as 
\begin{equation}\label{eq:empirical recons error 1}
    \hat{R}(P^\ell(C_n))=\frac{1}{n}\text{tr}(\mathbf{K})-\frac{1}{n^2}\sum^\ell_{j=1}\frac{\bm{\alpha}^\top_j\mathbf{K}^2\bm{\alpha}_j}{\hat{\lambda}_j}=\sum^n_{i=\ell+1}\hat{\lambda}_i
\end{equation}
and
\begin{equation}\label{eq:empirical recons error 2}
\hat{R}(P_m^\ell(C_n))=\frac{1}{n}\text{tr}(\mathbf{K})-\frac{1}{n}\sum_{j=1}^\ell\mathbf{u}_j^\top\mathbf{M}\mathbf{u}_j=\sum^n_{i=1}\hat{\lambda}_i-\sum^\ell_{i=1}\hat{\lambda}_{i,m},
\end{equation}
where $(\hat{\lambda}_j,\bm{\alpha}_j)_{j=1}^n$ and $(\hat{\lambda}_{j,m},\mathbf{u}_j)_{j=1}^m$ are the eigenvalue-vector pairs of $\frac{1}{n}\bf{K}$ and $\frac{1}{n}\mathbf{M}$ respectively. The number of principal components, $\ell$, is varied from $1$ to $m$. The results of the experiment are summarized in Figure~\ref{Fig:1}, where we observe that NY-KPCA has similar performance to that of EKPCA in terms of the empirical reconstruction error until a certain value of $\ell$ beyond which the performance seems to be surprisingly better than EKPCA. On the computational front, NY-KPCA is significantly faster than EKPCA with the latter having a runtime of 337 seconds. Similar behavior is observed for digit 2 and the results are presented in Figure~\ref{Fig:2}.
\begin{figure}[!tbp]
  \centering\vspace{-20mm}
  \begin{minipage}[b]{0.8\textwidth}
    \includegraphics[width=\textwidth]{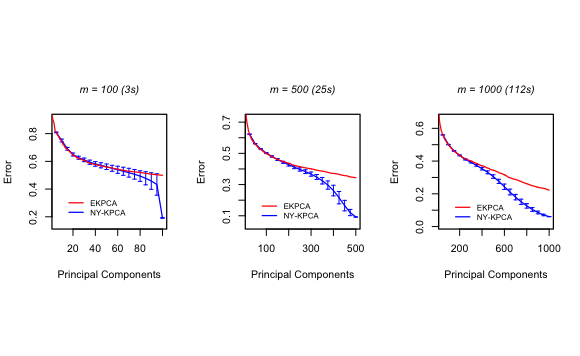}
  \end{minipage}
  \vspace{-15mm}
  \caption{Empirical reconstruction error of EKPCA and average empirical reconstruction error of 100 repetitions of NY-KPCA on digit 5 versus number of principal components $\ell$; error bars represent $\pm 2$ standard deviations. Runtime in seconds is given in parentheses next to the number of Nystr\"om subsamples $m$. Runtime for EKPCA is 337 seconds.}
  \vspace{-4mm}
  \label{Fig:1}
\end{figure}
\begin{figure}[!tbp]
  \centering\vspace{0mm}
  \begin{minipage}[b]{0.8\textwidth}
    \includegraphics[width=\textwidth]{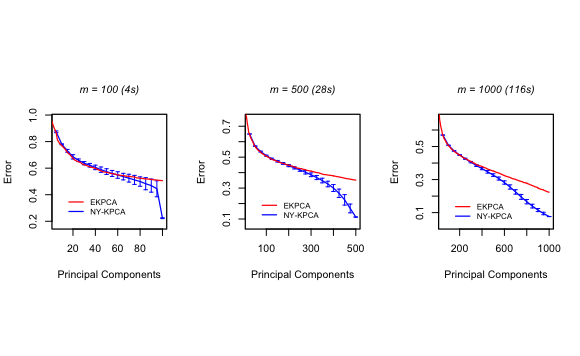}
  \end{minipage}
  \vspace{-15mm}
  \caption{Empirical reconstruction error of EKPCA and average empirical reconstruction error of 100 repetitions of NY-KPCA on digit 5 versus number of principal components $\ell$; error bars represent $\pm 2$ standard deviations. Runtime in seconds is given in parentheses next to the number of Nystr\"om subsamples $m$. Runtime for EKPCA is 478 seconds.}
  \vspace{-1mm}
  \label{Fig:2}
\end{figure}
\section{Proofs}\label{Sec:Proofs}
In this section, we present the proofs.
\subsection{Proof of Proposition~\ref{ny-kpca soln}}\label{sec:ny-kpca soln}
Define $$Z_n:\Hk\rightarrow\R^n,\,f\mapsto\left(f(X_1),\ldots,f(X_n)\right)^\top\,\,\text{and}\,\,\tilde{Z}_m:\Hk\rightarrow\R^m,\,f\mapsto\left(f(\tilde{X}_1),\ldots,f(\tilde{X}_n)\right)^\top.$$
The adjoint of $\tilde{Z}_m$ \citep{Smale-07} is given by
$$\tilde{Z}_m^*:\R^m\rightarrow\Hk,\,\alpha\mapsto\sum_{i=1}^m\alpha_ik(\cdot,\tilde{X}_i).$$
Thus, any $f\in\Hk_m$ may be written as $\tilde{Z}_m^*\alpha$, for some $\alpha\in\R^m$ and so
$$\inner{f}{C_nf}_\Hk=\frac{1}{n}\inner{\tilde{Z}_m^*\alpha}{Z_n^*Z_n\tilde{Z}_m^*\alpha}_\Hk=\frac{1}{n}\alpha^\top\tilde{Z}_mZ_n^*Z_n\tilde{Z}_m^*\alpha,$$
where we used $Z^*_nZ_n=\frac{1}{n}C_n$. It is easy to verify that $Z_n\tilde{Z}^*_m=\mathbf{K}_{nm}$ and $\tilde{Z}_mZ^*_n=\mathbf{K}_{mn}$. Therefore, 
(\ref{NY-KPCA problem defin}) can be written as
\begin{equation}\label{ny-kpca rewrite}
   \arg\sup \left\{\frac{1}{n}\alpha^\top\mathbf{K}_{mn}\mathbf{K}_{nm}\alpha:\alpha^\top\mathbf{K}_{mm}\alpha=1\right\}.
\end{equation}
Letting $\mathbf{u}=\mathbf{K}_{mm}^{1/2}\alpha$ simplifies the constraint in (\ref{ny-kpca rewrite}) to $\mathbf{u}^\top\mathbf{u}=1$, and we write (\ref{ny-kpca rewrite}) as
$$\arg\sup\left\{\frac{1}{n}\mathbf{u}^\top\mathbf{K}_{mm}^{-1/2}\mathbf{K}_{mn}\mathbf{K}_{nm}\mathbf{K}_{mm}^{-1/2}\mathbf{u}:\mathbf{u}^\top\mathbf{u}=1\right\}.$$ 
The solution to the above problem is the unit eigenvector of $\frac{1}{n}\mathbf{K}_{mm}^{-1/2}\mathbf{K}_{mn}\mathbf{K}_{nm}\mathbf{K}_{mm}^{-1/2}$ corresponding to its largest eigenvalue.  Denoting this eigenvector as $\mathbf{u}_{1,m}$, we obtain a function $\hat{\phi}_{1,m}\in\Hk$ solving the NY-KPCA problem in (\ref{NY-KPCA problem defin}) via $\hat{\phi}_{1,m}=\tilde{Z}_m^*\mathbf{K}_{mm}^{-1/2}\mathbf{u}_{1,m}.$
\subsection{Proof of \eqref{Eq:equiv}}
\label{sec:rewrite recons err}
Note that
\begin{eqnarray}
R(P)&{}={}&\bb{E}\norm{(I-P)k(\cdot,X)}_\Hk^2=\bb{E}\inner{(I-P)k(\cdot,X)}{(I-P)k(\cdot,X)}_\Hk \nonumber \\
&{}={}&\bb{E}\inner{(I-P)k(\cdot,X)}{k(\cdot,X)}_\Hk=\bb{E}\inner{(I-P)}{k(\cdot,X)\oh k(\cdot,X)}_\HSH, \label{rewrite recons err 1}
\end{eqnarray}  
where we used $\inner{Bf}{g}_\Hk=\inner{B}{f\oh g}_\HSH$ and  $(I-P)^2=(I-P)$ in (\ref{rewrite recons err 1}). Since $k$ is bounded, it follows that
$$\bb{E}\inner{(I-P)}{k(\cdot,X)\oh k(\cdot,X)}_\HSH=\inner{(I-P)}{\bb{E}[k(\cdot,X)\oh k(\cdot,X)]}_\HSH.$$
The result follows by using the above in \eqref{rewrite recons err 1} and noting that
\begin{eqnarray*}
\inner{(I-P)}{C}_\HSH=\text{tr}\left((I-P)C\right)=\text{tr}\left(C^{1/2}(I-P)(I-P)C^{1/2}\right)
\\
=\inner{(I-P)C^{1/2}}{(I-P)C^{1/2}}_\HSH=\norm{(I-P)C^{1/2}}_\HSH^2,
\end{eqnarray*}
where we have used the invariance of trace under cyclic permutations. 

\subsection{Proof of Theorem \ref{main theorem}}\label{main theorem proof}
$(i)$ For $t>0$, we have 
\begin{eqnarray}
R_{C_n,\ell}^{nys}&{}={}&\norm{(I-P_m^\ell(C_n))C^{1/2}}_\HSH^2=\norm{(I-P_m^\ell(C_n))(C_n+tI)^{1/2}(C_n+tI)^{-1/2}C^{1/2}}_\HSH^2\nonumber\\
&{}\le{}&\norm{(I-P_m^\ell(C_n))(C_n+tI)^{1/2}}_\OPH^2\norm{(C_n+tI)^{-1/2}C^{1/2}}_\HSH^2.\label{Eq:nystrom}
\end{eqnarray}
We now bound the terms in \eqref{Eq:nystrom}. First, we have
\begin{eqnarray}
    \norm{(C_n+tI)^{-1/2}C^{1/2}}_\HSH^2&{}={}&\norm{(C_n+tI)^{-1/2}(C+tI)^{1/2}(C+tI)^{-1/2}C^{1/2}}_\HSH^2
    \nonumber \\
    &{}\le{}&\norm{(C_n+tI)^{-1/2}(C+tI)^{1/2}}_\OPH^2\norm{(C+tI)^{-1/2}C^{1/2}}_\HSH^2
    \nonumber \\
    &{}={}&\underbrace{\norm{(C_n+tI)^{-1/2}(C+tI)^{1/2}}_\OPH^2}_{(A)}\Cal{N}_C(t),
\end{eqnarray}
where we used the fact $\norm{(C+tI)^{-1/2}C^{1/2}}_\HSH^2=\text{tr}(C^{1/2}(C+tI)^{-1}C^{1/2})=\text{tr}((C+tI)^{-1}C)=:\Cal{N}_C(t).$
Next, we have
\begin{eqnarray}
\norm{(I-P_m^\ell(C_n))(C_n+tI)^{1/2}}_\OPH^2 &{}\le{}&2\underbrace{\norm{(I-P_m)(C_n+tI)^{1/2}}_\OPH^2}_{(B)}\nonumber\\
&{}{}&\qquad+2\underbrace{\norm{(P_m-P_m^\ell(C_n))(C_n+tI)^{1/2}}_\OPH^2}_{(D)}, \label{main thm decomp 1}
\end{eqnarray}
where $P_m=Z_m^*(\mathbf{K}_{mm})^{-1}Z_m$ is the orthogonal projector onto $\Hk_m$ (see Section \ref{ny-kpca subsection}). $(B)$ can be bounded as 
\begin{equation}\label{main thm bnd 3}
    (B)\le\underbrace{\norm{(I-P_m)(C+tI)^{1/2}}_\OPH^2}_{(B_1)}\underbrace{\norm{(C+tI)^{-1/2}(C_n+tI)^{1/2}}_\OPH^2}_{(B_2)},
\end{equation}
and $(D)$ as
\begin{eqnarray}
(D)&{}\stackrel{(*)}{=}{}&\norm{(I-P_m^\ell(C_n))P_m(C_n+tI)^{1/2}}_\OPH^2\nonumber\\
    &{}={}&\norm{(I-P_m^\ell(C_n))P_m(C_n+tI)P_m(I-P_m^\ell(C_n))}_\OPH \nonumber \\
    &{}\le{}&\norm{(I-P_m^\ell(C_n))P_mC_nP_m(I-P_m^\ell(C_n))}_\OPH 
    +t\,\norm{(I-P_m^\ell(C_n))P_m(I-P_m^\ell(C_n))}_\OPH,\nonumber\\
    &{}\stackrel{(**)}{\le}{}&\hat{\lambda}_{\ell+1,m}+t,\label{Eq:D}
\end{eqnarray}
where we used the facts that $\Cal{R}(P_m^\ell(C_n))\subset\Cal{R}(P_m)$ in $(*)$ and $P_m^\ell(C_n)$ projects onto the $\ell$-eigenspace of $P_mC_nP_m$ in $(**)$. $\hat{\lambda}_{\ell+1,m}$ can be bounded as
\begin{equation}\label{main thm eigen bnd 1}
    \hat{\lambda}_{\ell+1,m}\le|\hat{\lambda}_{\ell+1,m}-\hat{\lambda}_{\ell+1}|+\hat{\lambda}_{\ell+1}\stackrel{(\dagger)}{\le}\frac{1}{n}\norm{\tilde{\mathbf{K}}-\mathbf{K}}_{\OPRn}+\hat{\lambda}_\ell,
\end{equation}
where $(\dagger)$ follows from the Hoffman-Wiendladt inequality \citep{Bhatia-94}.  We may rewrite (\ref{main thm eigen bnd 1}) as
\begin{eqnarray}
\frac{1}{n}\norm{\tilde{\mathbf{K}}-\mathbf{K}}_{\OPRn}&{}={}&\frac{1}{n}\norm{Z_n(I-P_m)Z_n^*}_{\OPRn}\nonumber\\
&{}={}&\norm{(I-P_m)C_n(I-P_m)}_\OPH=\norm{C_n^{1/2}(I-P_m)C_n^{1/2}}_\OPH \nonumber \\
&{}\le{}&\norm{C_n^{1/2}(C+tI)^{-1/2}}_\OPH^2\norm{(C+tI)^{1/2}(I-P_m)}_\OPH^2 \nonumber \\
&{}\stackrel{(\ddagger)}{\le}{}&\norm{(C_n+tI)^{1/2}(C+tI)^{-1/2}}_\OPH^2\norm{(C+tI)^{1/2}(I-P_m)}_\OPH^2,\label{Eq:K}
\end{eqnarray}

where we used $$\norm{C_n^{1/2}(C+tI)^{-1/2}}_\OPH^2\le \norm{C_n^{1/2}(C_n+tI)^{-1/2}}_\OPH^2 \norm{(C_n+tI)^{1/2}(C+tI)^{-1/2}}_\OPH^2$$
and $\Vert C_n^{1/2}(C_n+tI)^{-1/2}\Vert_\OPH^2\le 1$ in $(\ddagger)$. The result follows by combining \eqref{Eq:nystrom}--\eqref{Eq:K} and employing Lemmas \ref{lem:1} and \ref{lem:2} for $(iii)$.
\vspace{2mm}\\
$(ii)$ The proof follows exactly as in $(i)$; however, we bound $\norm{(I-P_m)(C+tI)^{1/2}}_\OPH^2$ with Lemma \ref{lem:3} with $t_0=\frac{19\kappa}{n}\log\frac{2n}{\delta}$.

\begin{lemma}\label{lem:1}
For $\delta>0$, suppose $\frac{9\kappa}{n}\log\frac{n}{\delta}\le t\le\lambda_1$.  Then the following hold:
\begin{itemize}
    \item[(i)] $\Pb^n\left\{\sqrt{\frac{2}{3}}\le\norm{(C+tI)^{1/2}(C_n+tI)^{-1/2}}_\OPH\le\sqrt{2}\right\}\ge 1-\delta;$
    \item[(ii)] $\Pb^n\left\{\norm{(C+tI)^{-1/2}(C_n+tI)^{1/2}}_\OPH\le\sqrt{\frac{3}{2}}\right\}\ge 1-\delta;$
    \item[(iii)] $\Pb^n\left\{\hat{\lambda}_\ell+t\le\frac{3}{2}(\lambda_\ell+t)\right\}\ge 1-\delta.$
\end{itemize}
\end{lemma}
\begin{proof}
$(i)$ The result is quoted from Lemma 3.6 of \citep{Rudi-13} with $\alpha=\frac{1}{2}$.
\par\noindent
$(ii)$  This is a slight variation of $(i)$ and the proof idea follows that of Lemma 3.6 of \citep{Rudi-13} with $\alpha=\frac{1}{2}$. Note that
$$\norm{(C+tI)^{-1/2}(C_n+tI)^{1/2}}_\OPH=\norm{(C+tI)^{-1/2}(C_n+tI)(C+tI)^{-1/2}}_\OPH^{1/2}.$$
By defining $B_n=(C+tI)^{-1/2}(C-C_n)(C+tI)^{-1/2}$, we have
$$I-B_n=(C+tI)^{-1/2}\left((C+tI)-C+C_n\right)(C+tI)^{-1/2}=(C+tI)^{-1/2}(C_n+tI)(C+tI)^{-1/2}$$
and therefore
\begin{equation}\label{B_n decomp}
    \norm{(C+tI)^{-1/2}(C_n+tI)^{1/2}}_\OPH=\norm{I-B_n}_\OPH^{1/2}\le\left(1+\norm{B_n}_\OPH\right)^{1/2}.
\end{equation}
It follow from the proof of Lemma 3.6 of \citep{Rudi-13} that for $\frac{9\kappa}{n}\log\frac{n}{\delta}\le t$, 
\begin{equation}\label{B_n bound}
    \Pb^n\left\{\norm{B_n}_\OPH\le\frac{1}{2}\right\}\ge 1-\delta.
\end{equation}
Combining (\ref{B_n decomp}) and (\ref{B_n bound}) completes the proof.\\
$(iii)$ Since $\sqrt{\frac{2}{3}}\le\norm{(C+tI)^{1/2}(C_n+tI)^{-1/2}}_\OPH$ as obtained in $(i)$, it is equivalent (see \cite[Lemmas B.2 and 3.5]{Rudi-13}) to $C_n+tI\preceq\frac{3}{2}(C+tI)$. This implies (see \citealp{gohberg}) that $\lambda_k(C_n+tI)\le \lambda_k(\frac{3}{2}(C+tI))=\frac{3}{2}\lambda_k(C+tI)$ for all $k\ge 1$.
\end{proof}
\begin{lemma}[\citep{Rudi-15}, Lemma 6]\label{lem:2}
Suppose Assumption \ref{first kern assum} holds, and suppose for some $m<n$, the set $\{\tilde{X}_j\}_{j=1}^m$ is drawn uniformly from the set of all partitions of size $m$ of the training data, $\{X_i\}_{i=1}^n$.  For $t>0$ and any $\delta>0$ such that $m\ge\left(67\vee5\Cal{N}_{C,\infty}(t)\right)\log\frac{4\kappa}{t\delta}$, we have
$$\Pb^n\left\{\norm{(I-P_m)(C+tI)^{1/2}}_\OPH^2\le3t\right\}\ge1-\delta,$$
where $P_m$ is the orthogonal projector onto $\Hk_m=\emph{span}\{k(\cdot,\tilde{X}_j)|j\in[m]\}$.
\end{lemma}
\begin{lemma}[\citep{Rudi-15}, Lemma 7]\label{lem:3}
Suppose Assumption \ref{first kern assum} holds.  Let $(\hat{l}_i(s))_{i=1}^n$ be the collection of approximate leverage scores.  Letting $N:=\{1,...,n\}$, for $t>0$ define $p_t$ as the distribution over $N$ with probabilities $p_t(i)=\hat{l}_i(t)/\sum_{j=1}^n\hat{l}_j(t)$.  Let $\Cal{I}_m=\{i_1,...,i_m\}\subset N$ be a collection of indices independently sampled from $p_t$ with replacement. Let $P_m$ be the orthogonal projector onto $\Hk_m=\emph{span}\{k(\cdot,\tilde{X}_j)|j\in\Cal{I}_m\}$.  Additionally, for any $\delta>0$, suppose the following hold:
\begin{enumerate}
    \item There exists $T\ge1$ and $t_0>0$ such that for any $s\ge t_0$, $(\hat{l}_i(s))_{i=1}^n$ are $T-$approximate leverage scores with confidence $\delta$,
    \item $n\ge1655\kappa+223\kappa\log\frac{2\kappa}{\delta},$
    \item $t_0\vee\frac{19\kappa}{n}\log\frac{2n}{\delta}\le t\le\lambda_1,$
    \item $m\ge334\log\frac{8n}{\delta}\vee78T^2\Cal{N}_C(t)\log\frac{8n}{\delta}$.
\end{enumerate}
Then
$$\Pb^n\left\{\norm{(I-P_m)(C+tI)^{1/2}}_\OPH^2\le3t\right\}\ge1-2\delta.$$
\end{lemma}
\subsection{Proof of Corollary \ref{poly decay corollary}}\label{poly decay proof}
$(i)$ From Theorem \ref{main theorem} $(i)$ we have 
$$
R_{C,\ell}=
\sum_{i>\ell}\lambda_i\lesssim\sum_{i>\ell} i^{-\alpha}\lesssim\int_\ell^\infty x^{-\alpha}dx\lesssim\ell^{1-\alpha}=n^{-\theta(1-\frac{1}{\alpha})}.$$
Similarly,
$$
R_{C,\ell}=\sum_{i>\ell}\lambda_i\gtrsim\sum_{i>\ell} i^{-\alpha}\gtrsim\int_\ell^\infty x^{-\alpha}dx\gtrsim\ell^{1-\alpha}=n^{-\theta(1-\frac{1}{\alpha})}.$$
\\
$(ii)$ This is Theorem 3.2 of  \citep{Rudi-15} with $\alpha=\frac{1}{2}$, $r=\alpha$, $p=2$, and $\ell=n^{\frac{\theta}{\alpha}}$.  
\vspace{1mm}\\
$(iii)$  Theorem \ref{main theorem} $(iii)$ and Proposition \ref{N(t) bnd} yield
$$R_{C_n,\ell}^{nys}\lesssim_{\Pb^n}t^{-\frac{1}{\alpha}}n^{-\theta}+t^{1-\frac{1}{\alpha}}\le\begin{cases}t^{1-\frac{1}{\alpha}},\qquad\quad\hspace{1mm} t\ge n^{-\theta}\\
t^{-\frac{1}{\alpha}}n^{-\theta},\qquad t\le n^{-\theta}
\end{cases},$$
where $\frac{\log n}{n}\lesssim t\le\lambda_1$ and $m\gtrsim\Cal{N}_{C,\infty}(t)\log\frac{1}{t}$ with $\Cal{N}_{C,\infty}(t)=\sup_{x\in\X}\inner{k(\cdot,x)}{(C+tI)^{-1}k(\cdot,x)}_\Hk\lesssim\frac{1}{t}$.
First, consider the case when $t\ge n^{-\theta}$. This means 
$$R_{C_n,\ell}^{nys}\lesssim \inf\left\{t^{1-\frac{1}{\alpha}}: t\gtrsim \frac{\log n}{n}\vee n^{-\theta},m\gtrsim\frac{1}{t}\log\frac{1}{t}\right\}.$$ For $\theta<1$, we obtain
$$R_{C_n,\ell}^{nys}\lesssim \inf\left\{t^{1-\frac{1}{\alpha}}: t\gtrsim n^{-\theta},m\gtrsim\frac{1}{t}\log\frac{1}{t}\right\}\le n^{-\theta\left(1-\frac{1}{\alpha}\right)}$$
if $m\gtrsim n^\theta\log n$.
For $\theta\ge 1$, we obtain
$$R_{C_n,\ell}^{nys}\lesssim \inf\left\{t^{1-\frac{1}{\alpha}}: t\gtrsim \frac{\log n}{n},m\gtrsim\frac{1}{t}\log\frac{1}{t}\right\}\le \left(\frac{\log n}{n}\right)^{\left(1-\frac{1}{\alpha}\right)}$$
if $m\gtrsim \frac{n}{\log n}\log\frac{n}{\log n}$. Next, consider the case when $t\le n^{-\theta}$ which means
$$R_{C_n,\ell}^{nys}\lesssim \inf\left\{t^{-\frac{1}{\alpha}}n^{-\theta}: \frac{\log n}{n}\lesssim t\lesssim  n^{-\theta},m\gtrsim\frac{1}{t}\log\frac{1}{t}\right\}\le n^{-\theta\left(1-\frac{1}{\alpha}\right)}$$
when $\theta<1$ and $m\gtrsim n^\theta\log n$.\vspace{2mm}\\
$(iv)$ Theorem \ref{main theorem}$(iv)$ and Proposition \ref{N(t) bnd} yield
$$R_{C_n,\ell}^{nys}\lesssim_{\Pb^n}t^{-\frac{1}{\alpha}}n^{-\theta}+t^{1-\frac{1}{\alpha}}\le\begin{cases}t^{1-\frac{1}{\alpha}},\qquad\quad\hspace{1mm} t\ge n^{-\theta}\\
t^{-\frac{1}{\alpha}}n^{-\theta},\qquad t\le n^{-\theta}
\end{cases},$$
where $\frac{\log n}{n}\lesssim t\le\lambda_1$ and $m\gtrsim\Cal{N}_C(t)\log n\gtrsim t^{-\frac{1}{\alpha}}\log n$. The result follows by carrying out the analysis as in $(iii)$ for $\theta<1$ and $\theta\ge 1$.

\subsection{Proof of Corollary \ref{exp decay corollary}}\label{exp decay proof}
$(i)$ From Theorem \ref{main theorem} $(i)$ we have 
$$R_{C,\ell}=\sum_{i>\ell}\lambda_i\lesssim\sum_{i>\ell} e^{-\tau i}\lesssim\int_\ell^\infty e^{-\tau x}dx\lesssim e^{-\tau\ell}=n^{-\theta}$$
and 
$$R_{C,\ell}=\sum_{i>\ell}\lambda_i\gtrsim\sum_{i>\ell} e^{-\tau i}\gtrsim\int_{\ell+1}^\infty e^{-\tau x}dx\gtrsim e^{-\tau(\ell+1)}=e^{-\tau}n^{-\theta}.$$
\\
$(ii)$ Theorem \ref{main theorem} $(ii)$ and Proposition \ref{N(t) bnd exp} yield
$$R_{C_n,\ell}\lesssim_{\Pb^n}\left(n^{-\theta}+t\right)\log\frac{1}{t}\le\begin{cases}n^{-\theta}\log\frac{1}{t},\qquad\hspace{1mm} t\le n^{-\theta}\\
t\log\frac{1}{t},\qquad\quad\hspace{2mm} t\ge n^{-\theta}
\end{cases},$$
where $\frac{\log n}{n}\lesssim t\le\lambda_1$. For the case of $t\le n^{-\theta}$, we obtain
$$R_{C_n,\ell}\lesssim \inf\left\{n^{-\theta}\log\frac{1}{t}: \frac{\log n}{n}\lesssim t\le n^{-\theta}\right\}=n^{-\theta}\log n,$$ where the constraint is only valid for $\theta<1$. On the other hand, for $t\ge n^{-\theta}$, we obtain
$$R_{C_n,\ell}\lesssim \inf\left\{t\log\frac{1}{t}: t\gtrsim\frac{\log n}{n}\vee  n^{-\theta}\right\}=\frac{\log n}{n}\log\left(\frac{n}{\log n}\right)\le \frac{(\log n)^2}{n},$$
which holds for $\theta\ge 1$.
\\
$(iii)$ Arguing similarly as in $(ii)$, it follows that for $\theta<1$ and $m\gtrsim n^{\theta}\log n$, we obtain a rate of $n^{-\theta}\log n$ for $R^{nys}_{C_n,\ell}$. Similarly for $\theta\ge 1$ and $m\ge \frac{n}{\log n}\log\left(\frac{n}{\log n}\right)$, we obtain a rate of $n^{-1}(\log n)^2$. 
\\
$(iv)$ Arguing as in $(ii)$ and enforcing the restriction $m\gtrsim\log n\log\frac{1}{t}$ imposed by Theorem \ref{main theorem} $(ii)$ yields the result.
\section{Conclusions}
In this paper, we considered the problem of deriving an approximation to kernel PCA using Nystr\"om method. This latter approach seemingly overcomes some of the difficulties of 
other approaches based on random features. In particular, it allows to derive error estimates directly comparable to those typically considered to analyze the statistical properties of KPCA. 
Our results indicate the existence of regimes where computational gains can be achieved while preserving statistical accuracy. These results parallel recent findings in supervised learning and are among the first of this kind for unsupervised learning. 

Our study opens a number of possible questions. For example, still for KPCA, it would be interesting to understand the properties of Nystr\"om sampling in combination with iterative 
eigensolvers, both batch (e.g., the power method) and stochastic (e.g., Oja's rule). The application of our approach to other spectral methods, such as those used in graph and manifold learning,  would be interesting. Beyond PCA and spectral methods, our study naturally yields the question of which other learning problems can have analogous statistical and computational trade-offs. For example, it would be interesting to consider  applications of our approach to independence tests based on covariance and cross-covariance operators \citep{Gretton-08}, or mean embeddings \citep{Sriperumbudur-10a}.

\bibliographystyle{apalike}
\bibliography{Nys_NeurIPS_Bib}
\appendix
\section{Technical Results}

\begin{proposition}\label{N(t) bnd}
Suppose $\underbar{A}i^{-\alpha}\le\lambda_i\le\bar{A}i^{-\alpha}$ for $\alpha>1$ and $\underbar{A},\bar{A}\in(0,\infty)$.  The following holds:
$$\Cal{N}_C(t)\lesssim t^{-1/\alpha}.$$
\end{proposition}
\begin{proof}
We have
$$\Cal{N}_C(t)=\text{tr}\left((C+tI)^{-1}C\right)=\sum_{i\ge 1}\frac{\lambda_i}{\lambda_i+t}\le\sum_{i\ge 1}\frac{\bar{A}i^{-\alpha}}{\underbar{A}i^{-\alpha}+t}=\frac{\bar{A}}{\underbar{A}}\sum_{i\ge 1}\frac{i^{-\alpha}}{i^{-\alpha}+t\underbar{A}^{-1}}.$$
Let $u=t^{1/\alpha}\underbar{A}^{-1/\alpha}x\implies u^\alpha=t\underbar{A}^{-1}x^\alpha$ and $dx=t^{-1/\alpha}\underbar{A}^{1/\alpha}du$. Therefore,
$$\sum_{i\ge 1}\frac{i^{-\alpha}}{i^{-\alpha}+t\underbar{A}^{-1}}\le\int_0^\infty\frac{x^{-\alpha}}{x^{-\alpha}+t\underbar{A}^{-1}}\,dx=\int_0^\infty\frac{1}{1+t\underbar{A}^{-1}x^\alpha}dx=\left(\frac{\underbar{A}}{t}\right)^{1/\alpha}\int_0^\infty\frac{1}{1+u^\alpha}du.$$
Since $\frac{1}{1+u^\alpha}$ is decreasing in $\alpha$ on $u\in(0,\infty)$, we have
$$\frac{1}{1+u^\alpha}\le\frac{1}{1+u^2},\hspace{3mm}\text{if}\hspace{3mm}\alpha\ge 2.$$ 
So for $2\le\alpha$,
$$\left(\frac{\underbar{A}}{t}\right)^{1/\alpha}\int_0^\infty\frac{1}{1+u^\alpha}du\stackrel{<}{\sim}t^{-1/\alpha}\int_0^\infty\frac{1}{1+u^2}du=t^{-1/\alpha}\,\left[\tan^{-1}(u)|_0^\infty\right]=\frac{\pi}{2}t^{-1/\alpha},$$
implying 
$\Cal{N}_C(t)\lesssim t^{-1/\alpha}$. For $1<\alpha<2$, we obtain
$$t^{-1/\alpha}\int_0^\infty\frac{1}{1+u^\alpha}du\le t^{-1/\alpha}\sum_{k=0}^\infty\frac{1}{1+k^\alpha}\le t^{-1/\alpha}\left(1+\sum_{k=1}^\infty\frac{1}{k^\alpha}\right).$$
Since $1+\sum_{k=1}^\infty\frac{1}{k^\alpha}$ converges for $\alpha>1$, we obtain
$\Cal{N}_C(t)\lesssim t^{-1/\alpha}.$
\end{proof}
\begin{proposition}\label{N(t) bnd exp}
Suppose $\underbar{B}e^{-\tau i}\le\lambda_i\le\bar{B}e^{-\tau i}$ for $\tau>0$ and $\underbar{B},\bar{B}\in(0,\infty)$.  Let $\ell=\frac{1}{\tau}\log n^\theta$, $\theta>0$. Then
$$\Cal{N}_C(t)\lesssim\log\left(\frac{1}{t}\right).$$
\end{proposition}
\begin{proof}
We have
$$\Cal{N}_C(t)=\text{tr}\left((C+tI)^{-1}C\right)=\sum_{i\ge 1}\frac{\lambda_i}{\lambda_i+t}\le\frac{\bar{B}e^{-\tau i}}{\underbar{B}e^{-\tau i}+t}=\frac{\bar{B}}{\underbar{B}}\sum_{i\ge 1}\frac{1}{1+t\underbar{B}^{-1}e^{\tau i}}$$
$$\lesssim\int_0^\infty\frac{1}{1+t\underbar{B}^{-1}e^{\tau x}}dx=\left[x-\frac{1}{\tau}\log\left(t\underbar{B}^{-1}\,e^{\tau x}+1\right)\right]\Big|_0^\infty.$$
Since
$$x-\frac{1}{\tau}\log\left(t\underbar{B}^{-1}\,e^{\tau x}+1\right)=\frac{1}{\tau}\left(\log(e^{\tau x})-\log\left(t\underbar{B}^{-1}\,e^{\tau x}+1\right)\right)=\frac{1}{\tau}\log\left(t^{-1}\underbar{B}\frac{e^{\tau x}}{\,e^{\tau x}+t^{-1}\underbar{B}}\right),$$
evaluating
$$\frac{1}{\tau}\log\left(t^{-1}\underbar{B}\frac{e^{\tau x}}{\,e^{\tau x}+t^{-1}\underbar{B}}\right)\Big|_0^\infty$$
yields the result.
\end{proof}
\end{document}